\documentclass[11pt,letterpaper]{amsart}
\usepackage{hyperref}
\usepackage{graphicx, color}
\usepackage{mathtools}
\usepackage{enumerate,enumitem}
\usepackage{bbm}
\usepackage[margin=1in]{geometry}
\usepackage{amssymb}
\usepackage{subcaption}
\usepackage{float}
\usepackage{tikz}
\usepackage{yhmath}
\usepackage{multirow}
\usepackage{setspace}
\usepackage{comment}

\usepackage{algpseudocode, algorithm}
\newcommand*\Let[2]{\State #1 $\gets$ #2}
\algrenewcommand\algorithmicrequire{\textbf{Given:}}
\algrenewcommand\algorithmicensure{\textbf{Return:}}

\usepackage[flushleft]{threeparttable}
\usepackage{array,booktabs,makecell}

\newtheorem{theorem}{Theorem}[section]

\newtheorem{lemma}[theorem]{Lemma}
\newtheorem{corollary}[theorem]{Corollary}
\theoremstyle{definition}
\newtheorem{definition}[theorem]{Definition}

\theoremstyle{remark}
\newtheorem{remark}[theorem]{Remark}

\DeclareMathOperator{\argmin}{argmin}
\DeclareMathOperator{\sign}{sign}

\newcommand{\R}{\mathbb{R}}
\newcommand{\E}{\mathbb{E}}
\newcommand{\prob}[1]{\mathrm{P}\left[ #1 \right]}

\newcommand{\0}{\boldsymbol{0}}
\newcommand{\1}{\boldsymbol{1}}
\newcommand{\A}{\mathbf{A}}
\newcommand{\Ac}{\mathcal{A}}
\renewcommand{\a}{\mathbf{a}}
\renewcommand{\b}{\mathbf{b}}

\newcommand{\e}{\mathbf{e}}
\newcommand{\g}{\mathbf{g}}
\newcommand{\h}{\mathbf{h}}
\newcommand{\I}{\mathbf{I}}

\newcommand{\p}{\mathbf{p}}
\newcommand{\q}{\mathbf{q}}
\newcommand{\Q}{\mathbf{Q}}

\newcommand{\w}{\mathbf{w}}
\newcommand{\W}{\mathbf{W}}
\newcommand{\x}{\mathbf{x}}
\newcommand{\X}{\mathbf{X}}
\newcommand{\y}{\mathbf{y}}
\newcommand{\z}{\mathbf{z}}
\newcommand{\Z}{\mathbf{Z}}

\newcommand{\RSnote}[1]{\textcolor{red}{[{\em {\bf RS:} #1}]}}
\newcommand{\JMnote}[1]{\textcolor{blue}{[{\em {\bf JM:} #1}]}}
\newcommand{\rev}[1]{#1}

\begin{document}
\title{A simple approach for quantizing neural networks}
\author{Johannes Maly}
\address{ Department of Mathematics, LMU Munich \newline and Munich Center for Machine Learning (MCML)}
\email{maly@math.lmu.de}

\author{Rayan~Saab}
\address{Department of Mathematics and Hal{\i}c{\i}o\u{g}lu Data Science Institute, University of California San Diego}
\email{rsaab@ucsd.edu}
\maketitle

\begin{abstract}
    In this short note, we propose a new method for quantizing the weights of a fully trained neural network. A simple deterministic pre-processing step allows us to quantize network layers via \emph{memoryless scalar quantization} while preserving the network performance on given training data. On  one hand, the computational complexity of this pre-processing slightly exceeds that of state-of-the-art algorithms in the literature. On the other hand, our approach does not require any hyper-parameter tuning and, in contrast to previous methods, allows a plain analysis. We provide rigorous theoretical guarantees in the case of quantizing single network layers and show that the relative error decays with the number of parameters in the network if the training data behaves well, e.g., if it is sampled from suitable random distributions. The developed  method also readily allows the quantization of deep networks by consecutive application to single layers.

\end{abstract}


\section{Introduction}
An \emph{$L$-layer feedforward neural network}, $\Phi: \R^{N_0} \to \R^{N_L}$ is a function whose action on a vector $\x\in\mathbb{R}^{N_0}$ is given by
\begin{equation}\label{eq-mlp}
    \Phi(\x):=\varphi \circ A^{(L)}. \circ\cdots\circ \varphi \circ A^{(1)}(\x),
\end{equation}
where the \emph{activation function} $\varphi:\R \to \R$ acts entry-wise on vectors, and $A^{(\ell)}:\R^{N_{\ell-1}}\to \R^{N_\ell}$ are affine maps given by $A^{(\ell)}(\z):= \W^{(\ell)\top}\z+\b^{(\ell)}$. We call $\W^{(\ell)}\in\R^{N_{\ell-1}\times N_\ell}$ and $\b^{(\ell)}\in\R^{N_\ell}$ the \emph{weight matrix} and \emph{bias vector} associated with the $\ell$-th layer of $\Phi$. The $i$-th \emph{neuron} (without activation) of the $\ell$-th layer is then the map $\z \mapsto (\w_i^{(\ell)})^\top \z + b_i^{(\ell)}$, where $\w_i^{(\ell)}$ denotes the $i$-th column of $\W^{(\ell)}$. In modern machine learning, neural networks have become the state-of-the-art tool for various tasks like speech recognition, autonomous driving, and games \cite{lecun2015deep,schmidhuber2015deep,goodfellow2016deep}.
Nevertheless, such networks tend to require a large number of layers, and a large number of parameters $N_\ell$ per layer. As a result, they are associated with high computational costs, both in storage/memory and in power usage. In order to reduce these costs, one approach is to use coarsly quantized parameters, i.e., quantized weights of the neural network (see  \cite{guo2018survey,  deng2020model, gholami2021survey}). This can be achieved  either by restricting the elements of $\W^{(\ell)}$ and $\b^{(\ell)}$ at training time to take on values from a discrete finite set, or by replacing them with elements from such a set after training \cite{krishnamoorthi2018quantizing}. The first approach entails \emph{quantization-aware training}, whereas the second involves \emph{post-training quantization} and is the focus of our work. In this context, quantization consists of replacing the, e.g.,  32-bit floating point numbers that constitute the weights of an already trained neural network  with coarsly quantized counterparts that can be represented with many fewer bits. The challenge lies in not degrading the performance of the network by doing so. 

To accomplish this task, one can progressively approach the problem one layer at a time, quantizing each neuron (column of $\W^{(\ell)}$) in the layer  before advancing to the next layer. \rev{Ignoring the bias terms for the moment and considering, for example, the first layer of the neural network of width $N_1$, one can select an appropriate \emph{alphabet} $\mathcal{A}$ and devise a map
\begin{align*}
 \mathcal{Q}:\quad   &\R^{N_0} \to \mathcal{A}^{N_0}   \\
 &\w\mapsto \q
\end{align*}

that respects $\w^\top \X \approx  \q^\top \X$ or equivalently $\X^\top \w \approx  \X^\top \q$, where $\X \in \R^{N_0\times m}$ is a matrix with $m$ training samples as its columns. Defining  $B:= \log_2(|\mathcal{A}|)$, each quantized neuron $\q$ from among the $N_1$ neurons  in the first layer can now be represented using $B N_0$ bits.} Variants of this approach have been explored recently (e.g., \cite{banner2018post, choukroun2019low, zhao2019improving, li2021brecq,lybrand2021greedy, zhang2022post}), including in the nascent literature that seeks rigorous theoretical guarantees for neural network quantization (e.g., \cite{lybrand2021greedy, zhang2022post, yin2019understanding, gunturk2021approximation}). 

While this general approach seems to work reasonably well in practice, and includes recent algorithms with theoretical guarantees, there are  some important challenges associated with it. First, an appropriate alphabet $\mathcal{A}$ must be chosen for each layer so that there even exists $\q \in \mathcal{A}^{\rev{N_0}}$ with $\X^\top \w \approx  \X^\top \q$. Second, the obvious approach once such an $\mathcal{A}$ is chosen, consists in finding $\q \in \mathcal{A}^{\rev{N_0}}$ that minimizes the objective function $\|\X^\top (\w-\q)\|_2$. However, this constitutes  an integer program, so it is generally NP-hard, and remains so for other objective functions. 
Despite these challenges, various ad-hoc computationally feasible approaches have been proposed, including \cite{banner2018post, choukroun2019low, zhao2019improving, li2021brecq}. 
\subsection{Related work} As already alluded to,  there has  been recent progress in developing computationally efficient algorithms with rigorous theoretical guarantees \cite{lybrand2021greedy, zhang2022post}. The authors of \cite{lybrand2021greedy} propose a greedy quantization algorithm based on \textit{noise-shaping} and analyze its performance in the case of a single layer neural network with Gaussian random training data, and they restrict their analysis to the case of the alphabet $\{-1,0,1\}$. Notably, the algorithm proposed in \cite{lybrand2021greedy} has computational complexity $\mathcal{O}(mN_0)$ per neuron, which is near optimal, as the size of the training data is $mN_0$. Subsequently \cite{zhang2022post} extends the analysis to more general distributions and alphabets. 
For example, if $\X$ is uniformly distributed in the ball of $\R^m$ of radius $r$, \cite{zhang2022post} shows that the  error of quantizing a neuron $\w \in \R^{N_0}$ satisfies
\begin{equation}\label{example-uniform-eq1}
\|\X^\top \w-\X^\top \q\|_2^2 \lesssim m r^2\log N_0,
\end{equation}
with high probability, where $\q \in\mathcal{A}^{N_0}$ contains the quantized weights. As a corollary, one can see that for generic vectors $\w$ that are independent of $\X$,
\begin{equation}
    \label{eq: error-bound-mixture}
    \frac{\|\X^\top \w-\X^\top \q\|_2^2}{\|\X^\top \w\|_2^2}\lesssim \frac{m\log N_0}{N_0}
\end{equation}
 with high probability. 
One issue with the theory in \cite{lybrand2021greedy, zhang2022post} is that it requires the largest element in $\mathcal{A}$ to be at least as large as $\|\w\|_\infty$. While this may seem innocuous for a single neuron $\w$, in practice the different columns in a weight matrix $\W^{(\ell)}$ may be bounded differently. As a result, one must either use different alphabets for each neuron, or accept a potentially large error bound. In practice, however, the numerical experiments in \cite{lybrand2021greedy, zhang2022post} use a single alphabet with a carefully tuned range to optimize the performance of the algorithm. In other words, they introduced an additional hyper-parameter that needs to be set a priori.

\subsection{Contribution}

In this work, we examine how one can reliably quantize a fully trained network $\Phi$ via \textit{memoryless scalar quantization}. Like \cite{lybrand2021greedy,zhang2022post} before, we restrict our analysis to quantizing a single network layer. We show that, surprisingly, a simple pre-processing step on $\w$ allows us to quantize the weights in a naive way and obtain theoretical guarantees as in~\eqref{eq: error-bound-mixture}. 
In contrast to the
noise-shaping approaches in \cite{lybrand2021greedy,zhang2022post}, the analysis is however remarkably simple. Moreover, even when quantizing full network layers no additional hyper-parameter tuning is required to boost the performance.
 The price we pay is that the pre-processing step is slightly more expensive in computation time. Let us mention that while our analysis is restricted to a single layer, the developed method readily lends itself to the quantization of deep networks as well by consecutive application to single layers.

\subsection{Notation} 

We abbreviate $[n] := \{1,\dots,n\}$, for $n \in \mathbb N$. We use  C, c, to denote absolute constants, while $a\lesssim b$ denotes $a\leq Cb$. Similarly, $a\gtrsim b$ means $a\geq Cb$ and $a \simeq b$ denotes $a \lesssim b \lesssim a$.  
\if{An $L$-layer multi-layer neural network, $\Phi: \R^{N_0} \to \R^{N_L}$ is a function whose action on a vector $\x\in\mathbb{R}^{N_0}$ is given by
\begin{equation}\label{eq-mlp}
    \Phi(\x):=\varphi \circ A^{(L)}. \circ\cdots\circ \varphi \circ A^{(1)}(\x)
\end{equation}
Here $\varphi:\R \to \R$ is an activation function acting entrywise on vectors, and $A^{(\ell)}:\R^{N_{\ell-1}}\to \R^{N_\ell}$ are affine maps given by $A^{(\ell)}(z):= \W^{(\ell)\top}\z+\b^{(\ell)}$. In turn,   $\W^{(\ell)}\in\R^{N_{\ell-1}\times N_\ell}$ and $\b^{(\ell)}\in\R^{N_\ell}$ are the weight matrix and bias vector associated with the $\ell$-th layer of the multi-layer neural network.}\fi
Henceforth, as justified by the observation  $\w^\top \x + b = \langle (\w, b), (\x,1)\rangle$, we will ignore the bias term $\b^{(\ell)}$ in \eqref{eq-mlp} as it  can simply be treated as an extra column of $\W^{(\ell)}$. Under this prerequisite, the $i$-th neuron of the $\ell$-th layer is defined as the map $\z \mapsto (\w_i^{(\ell)})^\top \z$, where $\w_i^{(\ell)}$ denotes the $i$-th column of $\W^{(\ell)}$. We define $\0$ and $\1$ to be the vector/matrix of zeros and ones, respectively (the dimensions will always be clear from the context). For a matrix $\W$, we denote the operator norm by $\| \W \|$ and the maximum entry in absolute value by $\| \W \|_\infty$. \rev{If $\A$ is an $n_1 \times n_2$  matrix and $J \subset [n_2]$}, we define the restricted kernel
\begin{align}
\label{eq:KerJdef}
    \ker_J(\A) := \{ \b \in \ker(\A) \colon b_i = 0 \text{ if } i \in J^c \}.
\end{align}

For $K\geq 1$, we define \emph{midrise} alphabets having $2K$ elements, as sets of the form
\begin{equation}\label{eq-alphabet-midrise}
    \mathcal{A}=\{\pm (k-1/2)\delta: 1\leq k\leq K, k\in\mathbb{Z}\}
\end{equation}
and, similarly, \emph{midtread} alphabets with $2K+1$ elements as sets of the form
\begin{equation}\label{eq-alphabet-midtread}
    \mathcal{A}=\{\pm k\delta: 0\leq k\leq K, k\in\mathbb{Z}\}
\end{equation}
where $\delta>0$ denotes the quantization step size. The simplest examples of such alphabets are the 1-bit alphabet $\{-1, 1\}$, and the ternary alphabet $\{-1,0,1\}$. The \emph{memoryless scalar quantizer} (MSQ) associated with an alphabet $\mathcal{A}$ is given by ${Q}:\mathbb{R}\rightarrow \mathcal{A}$ with
\begin{equation}\label{eq-MSQ}
    {Q}(z):=\arg\min_{p\in\mathcal{A}}|z-p|.
\end{equation}
For instance, the MSQ map ${Q}$ with $\Ac =\{-1,1\}$ is given by the two-valued $\sign$-function
\begin{align*}
    \sign(x) = \begin{cases} 1 & x \ge 0 \\ -1 & x < 0. \end{cases}
\end{align*}
In the following, we apply the quantizer $Q$ entry-wise to vectors and matrices.


\section{Quantizing a  network layer}

\begin{algorithm}[t] 
	\caption{\textbf{:}  \textbf{Neuron Preprocessing}} \label{alg:NeuronPreprocessing}
	\begin{algorithmic}[1]
		\Require{$\A_0 \in \R^{m\times n}$ ($n>m$), $\z_0 \in \R^n$, and $c \ge \|\z_0\|_\infty$}
		\Statex
		\State Initialize  $J_0 = \{ i \in [n] \colon \text{the $i$-th column of $\A_0$ is zero} \}$
        \State Define $\b \in \R^n$ via $\b_{J_0^c} = \0$ and $b_i = c - (z_0)_i$, for $i\in J_0$. Then $\b \in \ker(\A_0)$ and $|(z_0)_i + b_i| = c$ for $i \in J_0$
        \State Replace $\z_0$ with $\z_0 + \b$

        \State Initialize $k=0$

        \While{$\| |\z_k| - c\boldsymbol{1} \|_0 > m $ (which implies that $\ker_{J_k}(\A_k) \neq \{\0\}$, cf.\ Equation \eqref{eq:KerJdef})}
		\State Compute $\b \in \ker_{J_k}(\A_k)$, $\b \neq \0$
		\State Compute $\alpha \in \R$ with $\| \z_k + \alpha \b \|_\infty = c$
        \State{$\z_{k+1}$}={$\z_k + \alpha \b \in \R^{n}$}
		\State{$J_{k+1}$}={$J_k \cup \{ i \in [n] \colon |(\z_{k+1})_i| = c \}$}
		\State{$\A_{k+1}$}={$(\A_k)_{J_{k+1}^c} \in \R^{m \times n}$} (Matrix in which all columns indexed by $J_{k+1}$ are set to zero.)
		\Let {$k$}{$k+1$}
		\EndWhile
        \State $k_{\text{final}} = k$
		\Statex
		\Ensure{$\z_{k_{\text{final}}}$ for which $\A\z_{k_{\text{final}}} = \A\z_0$, $\| \z_{k_{\text{final}}} \|_\infty = c$, and $\| |\z_{k_{\text{final}}}| - c \1 \|_0 \le m$}
	\end{algorithmic}
\end{algorithm}

In this work, we consider uniform memoryless quantization.

\begin{definition}[Uniform $B$-bit quantizer]
\label{def:BbitQ}
    For any midrise or midtread alphabet $\Ac$ with $$\max_{q\in \Ac}{|q|}~=~1,$$ we define the quantization alphabet as $\Ac_c = c \cdot \Ac = \{-c,\dots,c\}$, for some suitable $c>0$. 
    
    If $|\Ac_c| = 2^B$, for $B \in \mathbb N$, then $\Ac_c$ can be encoded in $B$ bits, the worst-case distortion of $\Ac_c$ on $[-c,c]$ 
    \begin{align*}
        \delta_{\Ac_c} = \max_{z \in [-c,c]} |z - Q(z)|
    \end{align*}
    satisfies $\delta_{\Ac_c} = c 2^{-B}$ , and we call  the associated MSQ map $Q_c$ defined in \eqref{eq-MSQ},  a \emph{uniform $B$-bit quantizer}.
\end{definition}
We focus on quantization of single layer networks, i.e., the network $\Phi \colon \R^{N_0} \to \R^{N_1}$, $\Phi = \varphi \circ A$ consists of one layer. It is thus determined by the weight matrix $\W \in \R^{N_1\times N_0}$ of $A$. We furthermore assume that we have access to training data $\x_i, i \in [m]$ for which 
we have $\Phi(\x_i) = \y_i$.
We consider the overparametrized setting $N_0,N_1 \gg m$, i.e., there are far more trainable parameters than training samples. For convenience, we define the matrix of input data $\X = \big( \x_1, \dots, \x_m \big) \in \R^{N_0 \times m}$. For fixed $B \in \mathbb N$, i.e., $|\Ac_c| = 2^B$, our goal is thus to find a constant $c>0$ and a matrix $\Q \in \Ac_c^{N_0 \times N_1}$ such that $\Q^\top \X \approx \W^\top \X$. Lipschitz-continuity of the activation function $\varphi$ then guarantees $\varphi(\Q^\top \X)~\approx~\varphi(\W^\top \X)$.

\subsection{Quantizing a single neuron}

As a proof of concept, let us begin with the simpler case of quantizing one single neuron, i.e., the map $\z \mapsto \w^\top \z$. Given the data $\X$ we wish to construct $\q \in \Ac_c^{N_0}$ such that $\q^\top \X \approx \w^\top \X$, or equivalently, $\X^\top \q \approx \X^\top \w$. To this end, we define $\widehat c = \| \w \|_\infty$ and
\begin{align} \label{eq:wHatOneNeuron}
    \w^\sharp \in \arg\min_{\z \in \R^{N_0}} \| |\z| - \widehat c \1 \|_0,
    \qquad \text{s.t. } \X^\top \z = \X^\top \w \text{ and } \| \z \|_\infty \le \widehat c,
\end{align}
where $\| \cdot \|_0$ is not a norm but counts the number of non-zero entries and $|\cdot|$ is applied entry-wise.
The idea behind \eqref{eq:wHatOneNeuron} is to find a vector ${\w}^\sharp$ that mimics the action of $\w$ on the data, while at the same time having most of its entries exactly take on the values $\pm \widehat{c}$. Depending on the quantizer alphabet, the remaining entries can then be quantized more finely and the error can be easily bounded well. Unfortunately, the objective in \eqref{eq:wHatOneNeuron} is discrete and renders the optimization problem hard to solve. 

As a work-around, we propose Algorithm \ref{alg:NeuronPreprocessing} as an efficient procedure to compute substitute solutions. It is straight-forward to check that the algorithm (applied to $\A_0 = \X^\top$, $\z_0 = \w$, and $c = \widehat{c}$) stops after at most $N_0-m$ iterations and produces a vector $\widehat\w$ with $\X^\top \widehat\w = \X^\top \w$, $\| \widehat\w \|_\infty = \widehat{c}$, and $\| |\widehat\w| - \widehat c \1 \|_0 \le m$: indeed, Algorithm \ref{alg:NeuronPreprocessing} changes the input only along the kernel of $\X^\top$, keeps the $\ell_\infty$-norm of the iterates constant, and reduces the quantity $\| |\z_k| - c\boldsymbol{1} \|_0$ by at least one in each iteration. Although the computed solution $\widehat \w$ is not necessarily optimal in the sense of \eqref{eq:wHatOneNeuron}, it suffices for our purpose.
We now set 
\begin{align} \label{eq:qOneNeuron}
    \q = Q_{\widehat c}(\widehat\w),
\end{align}
where the MSQ $Q_{\widehat c}$ is applied entry-wise. We can  deduce the following result.

\begin{theorem} \label{thm:OneNeuron_Deterministic}
    Let $N_0 > m$, $\w \in \R^{N_0}$, and let $\X \in \R^{N_0\times m}$. Define the data complexity parameter
    \begin{align*}
        \Gamma (\X) = \sup_{\substack{T \subset [N_0] \\ |T| = m}} \| \X^\top |_T \|,
    \end{align*}
    where $\X^\top |_T$ denotes $\X^\top$ with all columns not indexed by $T$ set to zero.
    Let $\Q_{\widehat c}$ be a uniform $B$-bit quantizer as in Definition \ref{def:BbitQ}, for $\widehat c = \| \w \|_\infty$ and $B \in \mathbb N$. Then, if $\q$ is constructed via \eqref{eq:qOneNeuron}, where $\widehat \w$ is the output of Algorithm \ref{alg:NeuronPreprocessing}, we have that
    \begin{align} \label{eq:OneNeuronBound_Deterministic}
        \frac{\| \X^\top \w - \X^\top \q \|_2}{\| \X^\top \w \|_2}
        \le 2^{-B} \cdot \Gamma(\X) \cdot \frac{\sqrt{m} \| \w \|_\infty}{\| \X^\top \w \|_2}.
    \end{align}
\end{theorem}

\begin{proof}
    First note that, for any matrix $\X \in \R^{N_0\times m}$ with $N_0 > m$, the approximate solution $\widehat\w$ of \eqref{eq:wHatOneNeuron} computed by Algorithm \ref{alg:NeuronPreprocessing} consists of $N_0-m$ entries that are of magnitude $\widehat c = \| \w \|_\infty$ and has $m$ remaining entries of (possibly) smaller magnitude. Let us denote the set of these $m$ indices by $T \subset [N_0]$. Recall that $\X^\top \w = \X^\top \widehat \w$ and that the $B$-bit quantizer $Q_{\widehat c}$ has an entry-wise worst-case distortion of $2^{-B} \widehat c = 2^{-B} \| \w \|_\infty$ on the cube $[-\widehat c, \widehat c]^{N_0}$. By the definition of $\q$, it then follows that
    \begin{align*}
        \| \X^\top \w - \X^\top \q \|_2 
        &= \| \X^\top \widehat\w - \X^\top \q \|_2 
        = \| \X^\top|_T \cdot (\widehat\w - \q) \|_2
        \le \Gamma(\X) \cdot \sqrt{m} \| \widehat\w - \q \|_\infty \\
        &\le \Gamma(\X) \cdot 2^{-B} \sqrt{m} \| \widehat\w \|_\infty,
    \end{align*}
    where $\X^\top|_T$ denotes the matrix $\X^\top$ restricted to the columns indexed in $T$. 
    We thus have that
    \begin{align*}
        \frac{\| \X^\top \w - \X^\top \q \|_2}{\| \X^\top \w \|_2}
        \lesssim 2^{-B} \cdot \Gamma(\X) \cdot \frac{\sqrt{m} \| \widehat\w \|_\infty}{\| \X^\top \w \|_2}.
    \end{align*}
    The desired result follows trivially from the fact that $\| \widehat\w \|_\infty = \| \w \|_\infty$, cf.\ Algorithm \ref{alg:NeuronPreprocessing}.
\end{proof}

If the data $\X$ is, e.g., Gaussian, Theorem \ref{thm:OneNeuron_Deterministic} shows that the quantized neuron defined via $\q$ behaves similarly to the original neuron.

\begin{corollary} \label{cor:OneNeuron}
    Let $N_0 > m$, $\w \in \R^{N_0}$, and let $\X \in \R^{N_0\times m}$ have i.i.d.\ entries $X_{i,j} \sim \mathcal N (0,1)$. Let $\Q_{\widehat c}$ be a uniform $B$-bit quantizer as defined in Definition \ref{def:BbitQ}, for $\widehat c = \| \w \|_\infty$ and $B \in \mathbb N$. Then, with probability at least $1-4e^{-m \log(N_0)}$ on the draw of $\X$, if $\q$ is constructed via \eqref{eq:qOneNeuron}, where $\widehat \w$ is the output of Algorithm \ref{alg:NeuronPreprocessing}, we have that
    \begin{align} \label{eq:OneNeuronBound}
        \frac{\| \X^\top \w - \X^\top \q \|_2}{\| \X^\top \w \|_2}
        \lesssim 2^{-B} \frac{\sqrt{m \log(N_0)} \| \w \|_\infty}{\| \w \|_2}.
    \end{align}
\end{corollary}

\begin{proof}
    Since the non-zero entries of $\X^\top|_T$ form an $m\times m$ sub-matrix of an $m\times N_0$ Gaussian matrix, we get from \cite[Theorem 4.4.5]{vershynin2018high}, and a union bound over the ${N_0}\choose{m}$ submatrices of size $m\times m$, that
    \begin{align} \label{eq:ProofOneNeuronPartI}
        \Gamma(\X) \lesssim \sqrt{m\log(N_0)}
    \end{align}
    with probability at least $1-2e^{-m\log(N_0)}$.
    At the same time, the vector $\X^\top \w$ is Gaussian, so Lemma \ref{lem:GaussianNormConcentration} yields that with probability at least $1-2e^{-m}$
    \begin{align} \label{eq:ProofOneNeuronPartII}
        \| \X^\top \w \|_2 \gtrsim \sqrt{m} \| \w \|_2.
    \end{align}
    Combining \eqref{eq:ProofOneNeuronPartI} and \eqref{eq:ProofOneNeuronPartII} by a union bound and inserting them into \eqref{eq:OneNeuronBound_Deterministic}, we obtain that \eqref{eq:OneNeuronBound} holds with probability at least $1-4e^{-m \log(N_0)}$.
\end{proof}

A couple of comments are in order. First, the assumption that the entries of $\X$ in Corollary \ref{cor:OneNeuron} are standard Gaussian is only for ease of exposition. Indeed, the conclusions of the corollary hold for any (e.g., subgaussian) distribution for which \eqref{eq:ProofOneNeuronPartI} and \eqref{eq:ProofOneNeuronPartII} are satisfied with appropriately high probability. Second, Corollary \ref{cor:OneNeuron} strongly resembles the state-of-the-art results \cite[Theorem 2]{lybrand2021greedy} and \cite[Section 2]{zhang2022post}. Its proof is, however, remarkably simpler since our quantization technique is not adaptive but relies on the single pre-processing step in \eqref{eq:wHatOneNeuron}. For a generic weight vector $\w \in \R^{N_0}$, i.e., $\| \w \|_2 \simeq \sqrt{N_0} \| \w \|_\infty$, the bound in \eqref{eq:OneNeuronBound} becomes
\begin{align*}
    \frac{\| \X^\top \w - \X^\top \q \|_2}{\| \X^\top \w \|_2}
    \lesssim 2^{-B} \sqrt{\frac{m \log(N_0)}{N_0}}.
\end{align*}
Being of the same form as the just mentioned results, cf.\ Equation \eqref{eq: error-bound-mixture} above, this is a meaningful estimate in the overparametrized regime where the number of parameters exceeds the number of training data points, i.e., $N_0 \gg m$. Let us also emphasize that if the activation function $\varphi$ is $L$-Lipschitz continuous, the bound in \eqref{eq:OneNeuronBound} directly extends to the concatenation of neuron and activation function and becomes
\begin{align}
    \frac{\| \varphi(\X^\top \w) - \varphi (\X^\top \q) \|_2}{\| \X^\top \w \|_2}
    \lesssim L 2^{-B} \frac{\sqrt{m \log(N_0)} \| \w \|_\infty}{\| \w \|_2}.
\end{align}

\begin{remark}
    Computing the complexity parameter $\Gamma(\X)$ that appears in Theorem \ref{thm:OneNeuron_Deterministic} is challenging in general. However, it can trivially be bounded by $\| \X \|$. If $\X$ is a frame with upper frame bound $C$, this implies that $\Gamma(\X) \le C$. In particular, if $\X \in \R^{m \times N_0}$ is a concatenation of $N_0/m$ frames with upper frame bound $C$ (assuming for simplicity that $m$ divides $N_0$), it is easy to check that $\Gamma(\X) \le C \sqrt{N_0/m}$. For $N_0 = \mathcal O(m^2)$, this leads to $\Gamma(\X) \lesssim \sqrt{m}$ and thus to the same bound as in~\eqref{eq:OneNeuronBound}. 
\end{remark}

\begin{remark} \label{rem:LinftyMin}
    In order to improve the bound in Theorem \ref{thm:OneNeuron_Deterministic} and Corollary \ref{cor:OneNeuron}, we can find a $\z$ that minimizes the $\ell_\infty$-norm among all vectors satisfying $\| |\z| -  \|\z\|_\infty \1 \|_0 \le m$ and $\X^\top \z= \X^\top \w$. Indeed, if the rows of $\X$ are in general position, Lemma \ref{lem:LinftyMinProperty} shows that any solution 
    \begin{align} \label{eq:wHatOneNeuron_linfty}
        \widehat\w \in \arg\min_{\z \in \R^{N_0}} \| \z \|_\infty, \qquad \text{s.t. } \X^\top \z = \X^\top \w
    \end{align}
     fulfills $\| |\widehat\w| - \widehat c \1 \|_0 \le m$, for $\widehat c := \| \widehat\w \|_\infty \le \| \w \|_\infty$. This means we can solve \eqref{eq:wHatOneNeuron_linfty} instead of using Algorithm \ref{alg:NeuronPreprocessing}, which, can be more efficient, depending on the ratio between $N_0$ and $m$, cf.\ Section \ref{sec:Complexity}. We emphasize, however, that Algorithm \ref{alg:NeuronPreprocessing}  does not require the rows of $\X$ to be in general position.
    
\end{remark}

One may wonder how the quantized neuron performs on data from outside of the training set. The following theorem is an improved version of \cite[Theorem 3]{lybrand2021greedy} and answers this question in the case of new data drawn from the span of the training set. 

\begin{theorem}
\label{thm: generalize}
        Let \(\X, \w\) and \(\q\) be as in Corollary \ref{cor:OneNeuron} and  suppose that \(N_0 > m\). Then with \rev{probability at least $1-4e^{-2m \log(N_0)}$} we have for any data point $\z$ that lies in the span of $\X$ 
        that
        \begin{align}\label{eq: generalization error}
            |\z^\top(\w-\q)|\lesssim 2^{-B} \left(\frac{  m  \sqrt{\log(N_0)}}{\sqrt{N_0} - \sqrt{m}}\right) \|\z\|_2\| \w \|_\infty.
        \end{align}
\end{theorem}

\begin{proof}
    Define the set $\X (B_2^{N_0}) = \{ \boldsymbol{\zeta} \in \R^{N_0} \colon \boldsymbol{\zeta} = \X \h \text{ with } \| \h \|_2 \le 1 $     \} which is a bounded subset of the span of the data points. Then, for any $\boldsymbol{\zeta} \in \X (B_2^{N_0})$ one has
    \begin{align*}
        |\boldsymbol{\zeta}^\top (\w - \q)| 
        = \Big| \sum_{i = 1}^m h_i \x_i^\top (\w - \q) \Big| 
        \le \| \h \|_2 \| \X^\top (\w - \q) \|_2
        \lesssim 2^{-B} m \sqrt{\log(N_0)} \| \w \|_\infty,
    \end{align*}
    where we first used the Cauchy-Schwarz inequality, then the bound for the numerator in Corollary \ref{cor:OneNeuron} in the second inequality, which holds with probability at least $1-2e^{-m\log(N_0)}$. For $\z$ defined as in the statement, let $\p = \alpha_\star \z$ with
    \begin{align*}
        \alpha_\star = \max_{\alpha \ge 0} \alpha,
        \qquad \text{s.t.\ } \alpha \z \in \X (B_2^{N_0}).
    \end{align*}
    By using that $\p \in \X (B_2^{N_0})$, any strictly positive lower bound on $\alpha_\star$ would then yield a bound on our quantity of interest in \eqref{eq: generalization error} via
    \begin{align*}
        |\z^\top (\w-\q)| = \frac{1}{\alpha_\star} |\p^\top (\w - \q)| 
        \lesssim \frac{2^{-B} m \sqrt{\log(N_0)} \| \w \|_\infty}{\alpha_\star}.
    \end{align*}
    All that remains is to find a suitable lower bound for $\alpha_\star$. Since $\z$ is in the span of $\X$, there exists $\h_\z \in \R^m$ with $\z = \X\h_\z$. Since $N_0 > m$ and $\X$ is Gaussian, the embedding is almost surely injective and $\h_\z$ is unique. Setting $\bar{\h}_\z = \frac{\h_\z}{\| \h_\z \|_2}$, we have that $\X \bar{\h}_\z = \frac{\z}{\| \h_\z \|_2}$ and $\| \bar{\h}_\z \|_2 = 1$ which implies that $\alpha_\star \ge \| \bar{\h}_\z \|_2^{-1}$. We can now estimate that
    \begin{align*}
        \frac{\| \z \|_2}{\| \h_\z \|_2} 
        = \| \X\bar{\h}_\z \|_2 
        \ge \min_{\| \boldsymbol{\zeta} \|_2 = 1} \| \X \boldsymbol{\zeta} \|_2
        \gtrsim \sqrt{N_0} - \sqrt{m},
    \end{align*}
    where the last inequality holds with \rev{probability at least $1-2e^{-m}$} (over the draw of $\X$) and follows from standard bounds on the singular values of Gaussian matrices, e.g., \cite[Theorem 4.6.1]{vershynin2018high}. 
    Consequently, we obtain with the same probability that $\alpha_\star \ge \| \bar{\h}_\z \|_2^{-1} \gtrsim \tfrac{\sqrt{N_0} - \sqrt{m}}{\|\z\|_2}$. The claim follows from a union bound over both events.
\end{proof}

\if{
\begin{remark}
    \JMnote{Here we wanted to add some remark on Gaussian $\z$ and on relative error. I remember that we could add an $\sqrt{N_0}$ to the denominator by estimating the relative error and using $\| \w \|_2 \sim \sqrt{N_0}$. Unfortunately, I'm not sure anymore how we removed one of the $\sqrt{m}$ factors in the enumerator...}
\end{remark}
\RSnote{Not sure how satisfactory this is, or if it is worth putting in: Consider the case when $\z$ is a randomly drawn Gaussian vector that is  in the range of $\X$ and independent of $\w$, with a variance that is comparable to the variance of each training data point, i.e.,  
\(\E[\|\z\|_2^2 | \X] =  N_0\), and suppose \(N_0 \gg m\). Then, if we further assume that $\X$ and $\w$ are independent, and that $\|\w\|_2^2 \approx N_0$ we have $|\z^\top \w|^2 \approx \|\w\|_2^2\approx N$, so that the relative  error satisfies $\frac{|\z^\top(\w-\q)|}{|\z^\top\w|} \lesssim 2^{-B} \left(\frac{  m  \sqrt{\log(N_0)}}{\sqrt{N_0} - \sqrt{m}}\right)$}
\JMnote{We can also leave it as it is}
}\fi

\subsection{Quantizing a network layer}

The main challenge in generalizing \eqref{eq:wHatOneNeuron}-\eqref{eq:qOneNeuron} to a whole layer is that each neuron of the layer has a different upper bound on the magnitude of its entries, i.e., a different $\widehat c$. There is, however, a simple way to deal with this. First, define
$\widehat C = \| \W \|_\infty$ where $\W$ is the weight-matrix with columns $\w_i \in \R^{N_0}$ corresponding to single neurons, for $i \in [N_1]$.
The value of $\widehat C$ corresponds to the maximum $\widehat c_i = \| \w_i \|_\infty$ of all single neurons $\w_i$. 
We now solve
\begin{align} \label{eq:wHatOneLayerRefined}
    \W^\sharp_{\widehat C} \in \arg\min_{\Z \in \R^{N_0\times N_1}} \| |\Z| - \widehat C \1 \|_0,
    \qquad \text{s.t. } \X^\top \Z = \X^\top \W \text{ and } \| \Z \|_\infty \le \widehat C.
\end{align}
Since the optimization in \eqref{eq:wHatOneLayerRefined}  decouples in the single neurons, the columns $\w^\sharp_{i,\widehat C}$ of $\W^\sharp_{\widehat C}$ can be computed separately via
\begin{align} \label{eq:wHatOneLayerRefinedSingleNeuron}
    \w^\sharp_{i,\widehat C} \in \arg\min_{\z \in \R^{N_0}} \| |\z| - \widehat C \1 \|_0,
    \qquad \text{s.t. } \X^\top \z = \X^\top \w_i \text{ and } \| \z \|_\infty \le \widehat C.
\end{align}
Algorithm \ref{alg:NeuronPreprocessing} applied to $\A_0 = \X^\top$, $\z_0 = \w_i$, and $c = \widehat{C}$ can be used to get approximate solutions $\widehat\w_{i,\widehat C}$ of \eqref{eq:wHatOneLayerRefinedSingleNeuron}.
Having obtained a matrix $\widehat\W_{\widehat C}$ with columns $\widehat\w_{i,\widehat C}$ by consecutively applying Algorithm \ref{alg:NeuronPreprocessing}, we can now define 
\begin{align} \label{eq:qOneLayer}
    \Q = Q_{\widehat C}(\widehat\W_{\widehat C}).
\end{align}
Since each column of $\widehat\W_{\widehat C}$ has at most $N_0-m$ entries that are smaller than $\widehat{C}$ in magnitude, it is straight-forward to extend Theorem \ref{thm:OneNeuron_Deterministic} and Corollary \ref{cor:OneNeuron} to the following results.

\begin{theorem} \label{thm:OneLayer_Deterministic}
	Let $N_0 > m$, $\W \in \R^{N_0 \times N_1}$, and let $\X \in \R^{N_0\times m}$. Let $\Q_{\widehat C}$ be a uniform $B$-bit quantizer as  in Definition \ref{def:BbitQ} with $\widehat C = \| \W \|_\infty$ and $B \in \mathbb N$. Then,  if $\Q$ is constructed via \eqref{eq:qOneLayer}, where the columns of $\widehat\W_{\widehat C}$ are computed by Algorithm \ref{alg:NeuronPreprocessing}, we have that
	\begin{align} \label{eq:OneLayerBound_Deterministic}
	\frac{\| \X^\top \W - \X^\top \Q \|_F}{\| \X^\top \W \|_F}
	\le  2^{-B} \cdot \Gamma(\X) \cdot \frac{\sqrt{N_1 m} \| \W \|_\infty}{\| \X^\top \W \|_F},
	\end{align}
	where $\Gamma(\X)$ is the data complexity parameter from Theorem \ref{thm:OneNeuron_Deterministic}.
\end{theorem}

\begin{proof}
    The result follows by applying the same reasoning as in the proof of Theorem \ref{thm:OneNeuron_Deterministic} to each of the columns $\q_i$ of $\Q$ independently, i.e.,
    \begin{align*}
        \| \X^\top \w_i - \X^\top \q_i \|_2 
        \le \Gamma(\X) \cdot 2^{-B} \sqrt{m} \| \widehat\w_i \|_\infty,
    \end{align*}
    for any $i \in [N_1]$.
    This yields
    \begin{align*}
        \frac{\| \X^\top \W - \X^\top \Q \|_F^2}{\| \X^\top \W \|_F^2}
        = \frac{ \sum_{i=1}^{N_1} \| \X^\top \w_i - \X^\top \q_i \|_2^2}{\| \X^\top \W \|_F^2}
        \le 2^{-2B} \cdot \Gamma(\X)^2 \cdot \frac{N_1 m \| \widehat\W_{\widehat C} \|_\infty^2}{\| \X^\top \W \|_F^2},
    \end{align*}
    and thus the claim since $\| \widehat\W_{\widehat C} \|_\infty = \| \W \|_\infty$.
\end{proof}

Along the lines of Corollary \ref{cor:OneNeuron} one obtains then the following.

\begin{theorem} \label{cor:OneLayer}
	Let $N_0 > m \ge \log(N_1)$, $\W \in \R^{N_0 \times N_1}$, and let $\X \in \R^{N_0\times m}$ have i.i.d.\ entries $X_{i,j} \sim \mathcal N (0,1)$. Let $\Q_{\widehat C}$ be a uniform $B$-bit quantizer as defined in Definition \ref{def:BbitQ}, for $\widehat C = \| \W \|_\infty$ and $B \in \mathbb N$. Then, with probability at least $1-4e^{\log(N_1)-m}$ on the draw of $\X$, if $\Q$ is constructed via \eqref{eq:qOneLayer}, where the columns of $\widehat\W_{\widehat C}$ are computed by Algorithm \ref{alg:NeuronPreprocessing}, we have that
	\begin{align} \label{eq:OneLayerBound}
	\frac{\| \X^\top \W - \X^\top \Q \|_F}{\| \X^\top \W \|_F}
	\lesssim  2^{-B} \frac{\sqrt{N_1 m \log(N_0)} \| \W \|_\infty}{\| \W \|_F}.
	\end{align}
\end{theorem}

\begin{proof}
    Since $\X^\top|_T$ is a Gaussian $m\times m$-submatrix, Theorem \ref{thm:OneLayer_Deterministic} and \eqref{eq:ProofOneNeuronPartI} yield with probability at least $1-2e^{-m \log(N_0)}$ that
    \begin{align}
    \label{eq:ProofOneLayerPartI}
        \frac{\| \X^\top \W - \X^\top \Q \|_F}{\| \X^\top \W \|_F}
	    \le  2^{-B} \cdot \Gamma(\X) \cdot \frac{\sqrt{N_1 m} \| \W \|_\infty}{\| \X^\top \W \|_F}
        \le 2^{-B} \cdot \frac{m\sqrt{N_1 \log(N_0)} \| \W \|_\infty}{\| \X^\top \W \|_F}.
    \end{align}
    Moreover, by applying Lemma \ref{lem:GaussianNormConcentration} for each $\X^\top \w_i$ and using a union bound, we obtain with probability at least $1-2N_1 e^{-m}$ that
    \begin{align} \label{eq:ProofOneLayerPartII}
        \| \X^\top \w_i \|_2 \gtrsim \sqrt{m} \| \w_i \|_2,
    \end{align}
    for all $i \in [N_1]$.
    Combining \eqref{eq:ProofOneLayerPartI} and \eqref{eq:ProofOneLayerPartII} by another union bound, we thus have with probability at least $1-4e^{\log(N_1)-m }$ that
    \begin{align*}
        \frac{\| \X^\top \W - \X^\top \Q \|_F}{\| \X^\top \W \|_F}
        \lesssim 2^{-B} \cdot \frac{m\sqrt{N_1 \log(N_0)} \| \W \|_\infty}{\sqrt{m} \| \W \|_F}
        \le 2^{-B} \frac{\sqrt{N_1 m \log(N_0)} \| \W \|_\infty}{\| \W \|_F}.
    \end{align*}
\end{proof}

A similar discussion as in the single neuron case applies. If the activation function $\varphi$ is $L$-Lipschitz continuous, then for any generic weight matrix $\W \in \R^{N_0\times N_1}$, i.e., $\| \W \|_F \simeq \sqrt{N_0 N_1} \| \W \|_\infty$, the bound in \eqref{eq:OneLayerBound} becomes
\begin{align}
    \frac{\| \varphi(\X^\top \W) - \varphi (\X^\top \Q) \|_2}{\| \X^\top \W \|_2}
    \lesssim L 2^{-B} \sqrt{\frac{m \log(N_0)}{N_0}}.
\end{align}
As soon as $m \ll N_0$ this guarantees a small quantization error of the network when evaluated on the available data. 

\subsection{Computational complexity}
\label{sec:Complexity}

As Lemma \ref{lem:PreProcessingComplexity} in the appendix shows, Algorithm \ref{alg:NeuronPreprocessing} requires a run time of $\mathcal O(m^3 N_0)$ per single neuron. This is, by a factor of $m^2$, more computationally intensive than the near-optimal guarantees $\mathcal O(m N_0)$ provided in \cite{lybrand2021greedy,zhang2022post}. Meanwhile, Algorithm \ref{alg:NeuronPreprocessing} has only one hyper-parameter, namely the bit-budget $B$, since the required quantizer range $c$ is automatically determined by $\w$ resp.\ $\W$. Moreover, we will now present two algorithmic modifications to reduce our computational complexity when $\X$ is in general position.

\subsubsection{First variation.} One can slightly adapt Algorithm~\ref{alg:NeuronPreprocessing} as follows:\\
\begin{enumerate}
    \item Define in the beginning $\A^{(0)} \in \R^{m\times (m+1)}$ and $\tilde\A^{(0)} \in \R^{m\times m}$ as
    \begin{align*}
        \A^{(0)} = \begin{pmatrix}
        | & & | \\ \a_1 & \cdots & \a_{m+1} \\ | & & |
        \end{pmatrix} = \left(\begin{array}{c | c}
        & | \\ \tilde\A^{(0)} & \a_{m+1} \\  & |
        \end{array} \right).
    \end{align*}
    By computing $(\tilde\A^{(0)})^{-1}$ and $\tilde\b = (\tilde\A^{(0)})^{-1} \a_{m+1}$, the first kernel vector can be obtained via $\b = ( \tilde\b^\top, -1)^\top \in \ker(\A^{(0)})$. \label{eq:InitStep} \\
    \item Compute $\alpha$ as in Algorithm \ref{alg:NeuronPreprocessing} but reduced to the first $(m+1)$ entries of $\z_0$. \label{eq:BeginWhile}\\
    \item Choose $j' \in [m+1]$ as the smallest index $j$ with $|(\z_0 + \alpha \b)_j| = c$ and define $J_1 = J_0 \cup \{j'\}$. Note that only the first $(m+1)$ entries of $\z_0$ are updated to get $\z_1$.\\
    \item Generate $\A^{(1)}$ from $\A^{(0)}$ by replacing the $j'$-th column with $\a_{m+2}$.\\
    \item If $j' = m+1$, set $(\tilde\A^{(1)}) = (\tilde\A^{(0)})$, compute $\tilde\b = (\tilde\A^{(1)})^{-1} \a_{m+2}$, and obtain $\b = ( \tilde\b^\top, -1)^\top \in \ker(\A^{(1)})$. 
    If $j' < m+1$, abbreviate $\a = \a_{m+2}-\a_{j'}$ and note that $\tilde\A^{(1)} = \tilde\A^{(0)} + \a \e_{j'}^\top$, where $\e_{j'} \in \R^{m}$ denotes the $j'$-th unit vector. The Woodbury identity then yields
    \begin{align*}
        (\tilde\A^{(1)})^{-1} 
        = (\tilde\A^{(0)})^{-1} + \frac{1}{1 + \e_{j'}^\top (\tilde\A^{(0)})^{-1} \a} (\tilde\A^{(0)})^{-1} \a \; \e_{j'}^\top (\tilde\A^{(0)})^{-1}
    \end{align*}
    and $\tilde\b = (\tilde\A^{(1)})^{-1} \a_{m+1}$. (The matrix $(\tilde\A^{(1)})$ is invertible since the columns of $\A_0$ are in general position.) \\ \label{eq:EndWhile}
    \item While keeping thorough track of index switches, repeat Steps \eqref{eq:BeginWhile}-\eqref{eq:EndWhile} until all columns $\a_{m+2},\dots,\a_n$ have been used.
\end{enumerate}

As Lemma \ref{lem:PreProcessingComplexity_Accelerated} shows, this accelerated procedure requires a run time of $\mathcal O(m^2 N_0)$ which differs from \cite{lybrand2021greedy,zhang2022post} only by a factor $m$. 

\subsubsection{Second variation.} One can use $\ell_\infty$-minimization, as per \eqref{eq:wHatOneNeuron_linfty} in Remark \ref{rem:LinftyMin} instead of applying Algorithm \ref{alg:NeuronPreprocessing}. Lemma \ref{lem:LinftyComplexity} shows that, if the rows of $\X$ are in general position, \eqref{eq:wHatOneNeuron_linfty} can be solved by interior point methods up to accuracy $\delta > 0$ in $\mathcal{O}(N_0^{2.5} \log(N_0/\delta))$ time. If $m$ is of the same order as $N_0$, i.e., $m = \theta N_0$ for some $\theta \in (0,1)$, this run time differs from \cite{lybrand2021greedy,zhang2022post} only by a factor $N_0^{\frac{1}{2}} \simeq m^{\frac{1}{2}}$ (up to log-factors). Note, however, that some adaptions are necessary when pre-processing a whole layer $\W$ via $\ell_\infty$-minimization. Indeed, to obtain one $\widehat C$ for $\W$ one would solve
\begin{align} \label{eq:wHatOneLayer_linfty}
    \widehat\W \in \arg\min_{\Z \in \R^{N_0 \times N_1}} \| \Z \|_\infty, \qquad \text{s.t. } \X^\top \Z = \X^\top \W.
\end{align}
However, as \eqref{eq:wHatOneLayer_linfty} entails minimizing the infinity norm for each neuron, it follows that for several neurons the strict inequality $\| \widehat\w_i \|_\infty < \widehat C := \| \widehat\W \|_\infty$ may hold. This implies that we cannot quantize these neurons using $Q_{\widehat{C}}$ and still use Lemma \ref{lem:LinftyMinProperty} to control the error.  
To resolve this issue, after solving \eqref{eq:wHatOneLayer_linfty} one can find for each of the $N_1$ neurons, 
\begin{align}
\label{opt:P0}
    \w_i^\star \in \arg\min_{\z \in \R^{N_0}}~ \a^\top \z \quad \text{subject to}\quad \left\{
    \begin{array}{cc}
          \|\z\|_\infty& \leq \widehat{C}\\
          \X^\top \z & = \X^\top \w
    \end{array}\right. ,
\end{align}
where $\a \in \R^{N_0}$ is an arbitrary vector such that $\begin{pmatrix}
\X ~|~ \a
\end{pmatrix} \in \R^{N_0 \times (m+1)}$ is still in general position. As \eqref{opt:P0} is also a linear program it can be solved in $\mathcal{O}(N_0^{2.5} \log(N_0/\delta))$ time (by \cite[Theorem 1.1]{van2020deterministic}).
Surprisingly, however, the minimizers $\w_i^\star$ of \eqref{opt:P0} all satisfy $\|\w_i^\star\|_\infty = \widehat{C}$ and $|\{ i \in [N_0] \colon |w_i^\star| = \widehat C \}| \ge N_0 - m$ as we will now argue.

    To see that $\|\w_i^\star\|_\infty = \widehat{C}$, suppose by way of contradiction that $\|\w_i^\star\|_\infty < \widehat{C}$. Then, we can select $\h$ with $\X^\top\h=0$, and $\a^\top \h <0$. There exists an $\alpha>0$, small enough such that $\|\w_i^\star+\alpha \h\|_\infty \leq \widehat{C}$, with $\X^\top(\z^*+\alpha \h) = \X^\top \w$, and $\a^\top(\w_i^\star+\alpha \h)<\a^\top \w_i^\star$. This contradicts the optimality of $\w_i^\star$.   

    Now, consider the following auxiliary optimization problem, which we only use to prove that $\w_i^\star$ of \eqref{opt:P0} satisfies $|\{ i \in [N_0] \colon |w_i^*| = \widehat C \}| \ge N_0 - m$:
    \begin{align}\label{opt:P2}
    \bar\w \in \arg\min_{\z \in \R^n}~ \|\z\|_\infty  \quad \text{subject to}\quad \left\{
    \begin{array}{cc}
          \X^\top \z & = \X^\top \w_i^\star \\ 
          \a^\top \z & = \a^\top \w_i^\star
    \end{array}\right. .
\end{align}

Notice that $\|\bar{\w}\|_\infty \leq \widehat{C}$ since $\w_i^\star$, which satisfies $\|\w_i^\star\|_\infty = \widehat{C}$, satisfies the constraints of \eqref{opt:P2}.
    In turn, this means that  $\bar{\w}$ satisfies the constraints of \eqref{opt:P0}. Moreover, as $\a^\top \bar{\w} = \a^\top \w_i^\star$ is the optimal value for \eqref{opt:P0}, it follows that $\bar{\w}$  minimizes \eqref{opt:P0} and, as such, must satisfy $\|\bar{\w}\|_\infty = \widehat{C}$, which is also achieved by $\w_i^\star$. Thus the optimal value for \eqref{opt:P2} is also $\|\bar{\w}\|_\infty = \widehat{C}$.
    Collecting these results we see that \eqref{opt:P0} and \eqref{opt:P2} have the same minimizers. 
    Now apply Lemma \ref{lem:LinftyMinProperty} to \eqref{opt:P2},  noting that the concatenated matrix consisting of $\X^T$ and $\a^\top$ is of size $(m+1)\times N_0$, to conclude that $|\{i \in [N_0]: |w^\star_i| = \widehat C \}| \geq N_0-m$.

\appendix

\section{Technical addendum}

It is well-known that the norm of $n$-dimensional Gaussian vectors strongly concentrates around $\sqrt{n}$. For the reader's convenience we recall this fact in the following lemma.

\rev{
\begin{lemma}[{\cite[Ch.\ 2]{wainwright2019high}}] \label{lem:GaussianNormConcentration}
   Let $\g \sim \mathcal N(\0,\I_{n\times n})$ be an $n$-dimensional standard Gaussian vector. Then, for any $\theta > 0$,
   \begin{align*}
       \prob{ | \| \g \|_2 - \sqrt{n} | \ge \theta} \le 2e^{-\frac{\theta^2}{8}}.
   \end{align*}
\end{lemma}
}

The next lemma proves the claim made in Remark \ref{rem:LinftyMin}, namely that $\ell_\infty$-minimization provides the same properties as Algorithm \ref{alg:NeuronPreprocessing} if the columns of $\X$ are in general position.

\begin{lemma} \label{lem:LinftyMinProperty}
    Let $m \le n$, let $\A\in\R^{m\times n}$ have columns in general position, i.e., any $m$ columns of $\A$ span $\R^m$, and let $\b \in \R^m$. Then any
    \begin{align} \label{eq:wHatOneNeuron_linftyII}
        \z^\star \in \arg\min\limits_{\z \in \R^{n}} \| \z \|_\infty, \qquad
        \text{s.t. } \b = \A \z
    \end{align}
    has the property that  $|\{i \in [n]: |z^\star_i| = \|\z^\star\|_\infty \}| \geq n-m+1$.
\end{lemma}

\begin{proof}
    Suppose that any collection of $m$ columns of $\A$ spans $\R^m$. Suppose further that $\z^\star$ solves \eqref{eq:wHatOneNeuron_linftyII} and that $T:=\{i \in [n]: |z^\star_i| = \|\z^\star\|_\infty \}$ has $|T| < n-m+1$, thus $|T^c| \geq m$. 
    Then there exists a non-zero vector $\boldsymbol{\eta}^{\varepsilon} \in \ker(\A)$ parametrized by $\varepsilon > 0$  with 
    \begin{align*}
        \boldsymbol{\eta}^{\varepsilon}_T = -\varepsilon \cdot \z^\star_T = -\varepsilon \cdot \sign(\z^\star_T) \cdot \|\z^\star\|_\infty \text{\quad and \quad}
        \A_T \boldsymbol{\eta}^{\varepsilon}_T = -\A_{T^c} \boldsymbol{\eta}^{\varepsilon}_{T^c},
     \end{align*}
     where $\A_T \in \R^{m \times |T|}$ and $\A_{T^c} \in \R^{m\times |T^c|} $ are the submatrices of $\A$ formed by the columns indexed by $T$ and $T^c$. Indeed, to construct such a vector, simply pick $\boldsymbol{\eta}^{\varepsilon}_T$ to satisfy the first  equation above and pick $S\subset T^c$ with $|S| =m$, then set 
     $$\boldsymbol{\eta}^{\varepsilon}_{S} =- \A_S^{-1} (\A_{T} \boldsymbol{\eta}^{\varepsilon}_T  ) \quad \text{and } \quad \boldsymbol{\eta}^{\varepsilon}_{T^c \setminus S}=0$$
     Now, notice that $\A (\z^\star+\boldsymbol{\eta}^{\varepsilon}) = \A \z^\star  = \b$,
     i.e., $\z^\star+\boldsymbol{\eta}^{\varepsilon}$ is feasible to \eqref{eq:wHatOneNeuron_linftyII}
     and 
     \begin{align*}
         \|\z^\star+\boldsymbol{\eta}^{\varepsilon}\|_\infty &=   \max\{   \|\z^\star_T+\boldsymbol{\eta}^{\varepsilon}_T\|_\infty, \|\z^\star_{S} + \boldsymbol{\eta}^{\varepsilon}_S\|_\infty , \|\z^\star_{T^c\setminus S}\|_\infty \}.
         \\&= \max\{  (1-\varepsilon) \|\z^\star_T\|_\infty, \|\z^\star_{S} + \varepsilon \A_S^{-1}\A_T \z^\star_T\|_\infty , \|\z^\star_{T^c\setminus S}\|_\infty \}.
    \end{align*}
    Since there is a non-zero gap between the magnitude of entries of $\z^\star$ on $T$ and $T^c$ respectively, by continuity there is an $\varepsilon>0$ small enough so that 
    $$(1-\varepsilon) \|\z^\star_T\|_\infty \geq \max\{\|\z^\star_{S} + \varepsilon \A_S^{-1}\A_T \boldsymbol{\eta}^{\varepsilon}_T\|_\infty,\|\z^\star_{T^c\setminus S}\|_\infty \} , $$ 
    and thus $\|\z^\star+\boldsymbol{\eta}^{\varepsilon}\|_\infty < \|\z^{\star}\|_\infty$. However, $\z^\star$ was defined as a minimizer of \eqref{eq:wHatOneNeuron_linftyII}, which is a contradiction. It follows that $T$ must satisfy $|T|\geq n-m+1$. 
\end{proof}

The final three lemmas formalize the claims made in Section \ref{sec:Complexity} by analyzing the computational complexity of Algorithm \ref{alg:NeuronPreprocessing} (Lemma \ref{lem:PreProcessingComplexity}), of the accelerated version of Algorithm \ref{alg:NeuronPreprocessing} described in Section \ref{sec:Complexity} (Lemma \ref{lem:PreProcessingComplexity_Accelerated}), and of the $\ell_\infty$-minimization described in Remark \ref{rem:LinftyMin} (Lemma \ref{lem:LinftyComplexity}).

\begin{lemma} \label{lem:PreProcessingComplexity}
   Let $m \le n$. For $\A_0 \in \R^{m\times n}$ and $\z_0 \in \R^n$, Algorithm \ref{alg:NeuronPreprocessing} computes an output $\z_{k_{\text{final}}}$ in $\mathcal{O}(m^3 n)$ time.
\end{lemma}

\begin{proof}
    The steps before the while-loop  require $\mathcal O(n)$ time since they only involve adding $n$-dimensional vectors.
    
    The only steps in the while-loop that are relevant for determining the computational complexity are (i) determining $\b \in \ker_{J_k}(\A)$ and (ii) computing $\alpha$. 
    
    First note that in (i) an arbitrary kernel element of the restricted matrix $A_{J_k}$ is needed. One thus can reduce $\A_k$ to $(m+1)$ non-zero columns before computing $\b$, which then requires $\mathcal{O}(m^3)$ time. Let us denote the subset of indices of these $(m+1)$ columns by $I \subset [n]$. 
    
    Furthermore, it is straight-forward to check that (ii) can be computed in $\mathcal{O}(m)$ time. One just determines $\alpha_i \in \R$ with $| (z_k)_i + \alpha_i b_i | = c$ and $|\alpha_i|$ minimal, for all $i \in I$, and then sets $\alpha = \argmin_{i \in I} |\alpha_i|$. Here it is important to note that the only relevant coordinates of $\b$ (and $\z_k$) are the entries indexed by $I$. 
    
    Since these computations are performed $(n-m)$-times in the worst case (in each iteration the quantity $\| |\z_k| - c\boldsymbol{1} \|_0$ is reduced by at least one), we obtain the claimed time complexity.
\end{proof}

\begin{lemma} \label{lem:PreProcessingComplexity_Accelerated}
   Let $m \le n$. For any $\z_0 \in \R^n$ and $\A_0 \in \R^{m\times n}$ with columns in general position, i.e., any $m$ columns of $\A$ span $\R^m$, the accelerated version of Algorithm \ref{alg:NeuronPreprocessing} described in Section \ref{sec:Complexity} outputs $\z_{k_{\text{final}}}$ in $\mathcal{O}(m^2 n)$ time.
\end{lemma}

\begin{proof}
    Note that there is only one full matrix inversion in Step \eqref{eq:InitStep} which costs $\mathcal O(m^3)$. In Steps \eqref{eq:BeginWhile}-\eqref{eq:EndWhile} only the computation of $\alpha$ --- complexity $\mathcal{O}(m)$ --- and matrix-vector multiplications of dimension $m$ --- complexity $\mathcal O(m^2)$ --- take place. Since these are repeated $(n-m)$-times, the overall complexity is $\mathcal O(m^3) + \mathcal O(m^2 (n-m)) = \mathcal O(m^2 n)$.
\end{proof}

\begin{lemma} \label{lem:LinftyComplexity}
   Let $m \le n$. For $\A \in \R^{m\times n}$ and $\y \in \R^n$, the minimization
   \begin{align} \label{eq:InfMin}
        \min_{\z \in \R^{n}} \| \z \|_\infty, \qquad \text{s.t. } \A\z = \y
    \end{align}
    can be solved by interior point methods up to accuracy $\delta > 0$ in $\mathcal{O}(n^\omega \log(n/\delta))$ time. Here $\mathcal{O}(n^\omega)$ is the time required to multiply two $n\times n$-matrices, with the best $\omega$ known to satisfy  $\omega < 2.5$.
\end{lemma}

\begin{proof}
    Note that \eqref{eq:InfMin} is equivalent to the linear program
    \begin{align} \label{eq:IntermediateSystem}
        \min_{\substack{\z \in \R^n, u \in \R_{\ge 0} }} u
        \qquad \text{s.t. } \begin{cases}
        \A\z = \y \\ u\1 - \z \ge \0 \\ \z + u\1 \ge \0
        \end{cases}.
    \end{align}
    By introducing the auxiliary variables $\w_+ = u\1 - \z \in \R^n$ and $\w_- = u\1 + \z \in \R^n$, and denoting $\w = (\w_+^\top, \w_-^\top, u)^\top \in \R^{2n+1}$, we can re-write \eqref{eq:IntermediateSystem} as
    \begin{align} \label{eq:IntermediateSystemII}
        \min_{\w \in \R^{2n+1} } \e_{2n+1}^\top \w 
        \qquad \text{s.t. } \begin{cases}
        \tilde\A \w = \y \\ \w \ge \0
        \end{cases},
    \end{align}
    where $\e_{2n+1}$ is the $(2n+1)$-th unit vector and
    \begin{align*}
        \tilde\A = \begin{pmatrix}
        -\A & \A & \0 
        \end{pmatrix} \in \R^{m \times (2n+1)}.
    \end{align*}
    The claim now follows by applying \cite[Theorem 1.1]{van2020deterministic} to \eqref{eq:IntermediateSystemII}.
\end{proof}

\section*{Acknowledgments}
RS was supported in part by National Science Foundation Grant DMS-2012546, and by a Simons Fellowship.

\bibliography{references}
\bibliographystyle{plain}

\end{document}